\documentclass{article}

\usepackage{arxiv}
\usepackage[utf8]{inputenc}
\usepackage[T1]{fontenc}
\usepackage{amsmath,amssymb,amsthm,mathtools}
\usepackage{booktabs}
\usepackage{microtype}
\usepackage[round]{natbib}
\usepackage[colorlinks=true,citecolor=blue,linkcolor=blue,urlcolor=blue]{hyperref}

\newcommand{\Fone}{\texorpdfstring{\ensuremath{F_1}}{F1}}
\renewcommand{\headeright}{Preprint}
\renewcommand{\undertitle}{Preprint}
\renewcommand{\shorttitle}{Exact Rank and CC-Dimension Lower Bounds for \Fone{}}

\newtheorem{theorem}{Theorem}[section]

\newtheorem{lemma}[theorem]{Lemma}
\newtheorem{corollary}[theorem]{Corollary}
\theoremstyle{definition}

\theoremstyle{plain}
\newtheorem{remark}[theorem]{Remark}

\newcommand{\R}{\mathbb{R}}
\newcommand{\cY}{\mathcal{Y}}
\newcommand{\cU}{\mathcal{U}}
\newcommand{\one}{\mathbf{1}}
\DeclareMathOperator{\rank}{rank}
\DeclareMathOperator{\range}{range}
\DeclareMathOperator{\affdim}{affdim}
\DeclareMathOperator{\aff}{aff}
\DeclareMathOperator{\CCdim}{CCdim}
\DeclareMathOperator{\relint}{relint}

\title{Exact Rank and Convex Calibration Dimension Lower Bounds
for the Multi-Label \Fone{} Loss}
\author{
  Mingyuan Zhang\\
  Independent Researcher\\
  \href{mailto:myz@alumni.upenn.edu}{myz@alumni.upenn.edu}
}
\date{}

\hypersetup{
  pdftitle={Exact Rank and Convex Calibration Dimension Lower Bounds for the Multi-Label F1 Loss},
  pdfauthor={Mingyuan Zhang},
  pdfsubject={Machine Learning},
  pdfkeywords={multi-label classification, F1 loss, convex calibration dimension, loss-matrix rank}
}

\begin{document}
\maketitle

\begin{abstract}
The instance-wise \Fone{} measure is a central performance measure for
multi-label classification.  For a problem with $s$ labels, it defines a
$2^s\times 2^s$ loss matrix.  Previous work exhibited $s^2+1$-coordinate
affine and shifted low-rank representations and used them to construct
quadratic-dimensional convex calibrated surrogates.  We determine the exact
rank.  Under the convention $F_1(\varnothing,\varnothing)=1$, the
\Fone{} score matrix, the shifted loss matrix, and the unshifted loss matrix all
have rank $s^2-s+2$, while the column-affine dimension of the loss is
$s^2-s+1$.  The proof factors the nonempty score matrix through subset-incidence
matrices and a positive-definite Cauchy matrix.

Exact rank does not, by itself, lower-bound the dimension of an arbitrary
convex calibrated surrogate.  We therefore analyze the Bayes geometry of \Fone{}
directly.  We construct a distribution for which precisely all supersets of a
fixed core label set are Bayes optimal, and show that the corresponding active
loss columns, restricted to the witness support, have affine dimension $hn$,
where
$n=s-\lfloor s/3\rfloor$ and
$h=\lceil\sqrt{s\lfloor s/3\rfloor}\rceil-1$.  Applying the feasible-subspace
lower bound for convex calibration dimension gives
\[
  \CCdim(L^{F_1})
  \ge \left(\frac{2}{3\sqrt{3}}-o(1)\right)s^2.
\]
Together with the quadratic upper bound, this establishes
$\CCdim(L^{F_1})=\Theta(s^2)$.
\end{abstract}

\keywords{multi-label classification \and \Fone{} loss \and
  convex calibration dimension \and loss-matrix rank}

\section{Introduction}

In multi-label classification, an outcome and a prediction are subsets of a
ground set of $s$ labels.  The instance-wise \Fone{} score rewards overlap while
balancing precision and recall.  Its dependence on both the intersection and
the two set cardinalities makes it nondecomposable across labels.  Although the
output space has cardinality $2^s$, Bayes-optimal \Fone{} prediction can be
performed from only a quadratic number of conditional-distribution statistics
\citep{dembczynski2011exact,waegeman2014bayes}.  This quadratic structure also
underlies consistent plug-in algorithms \citep{dembczynski2013optimizing} and
convex calibrated surrogates \citep{nowak2019sharp,zhang2020convex}.

Earlier work exposed this structure in two complementary ways.
\citet{nowak2019sharp} gave $s^2+1$-coordinate affine decompositions and a
consistent least-squares surrogate.  Writing $L^{F_1}$ for the
$2^s\times2^s$ \Fone{} loss matrix and $J$ for the all-ones matrix,
\citet{zhang2020convex} showed $\rank(L^{F_1}-J)\le s^2+1$ and constructed
convex calibrated surrogates of the same dimension.  These results left two
natural questions:
\begin{enumerate}
  \item What is the exact rank of the \Fone{} loss matrix?
  \item Must every convex calibrated surrogate for \Fone{} have quadratic
        prediction dimension?
\end{enumerate}

We answer both questions.  First, we show that the shifted loss has exact rank
$s^2-s+2$ under the empty-set convention used by
\citet{zhang2020convex}.  We additionally compute the ordinary rank and affine
dimension of the loss columns.  Second, we prove a quadratic lower bound for
the convex calibration dimension.  The second result requires substantially
more than matrix rank: rank controls the dimension of linear factorizations and
provides general upper bounds, but it is not a lower bound for arbitrary convex
surrogates \citep{ramaswamy2016ccdim}.

We construct one hard conditional distribution with many independent Bayes
ties.  A fixed set $T$ of labels is
made sufficiently likely that every Bayes-optimal prediction must contain
$T$.  The remaining marginals are tuned so that every superset of $T$ has the
same expected \Fone{} score.  Varying the optional labels and their cardinalities
then creates a quadratic-dimensional active face of the Bayes-risk polytope.
The general feasible-subspace bound then converts this local geometry into a
lower bound for every convex calibrated surrogate
\citep{ramaswamy2016ccdim}.

\paragraph{Contributions.}
For every $s\ge1$, we prove
\[
  \rank(L^{F_1}-J)=\rank(L^{F_1})=s^2-s+2,
  \qquad
  \affdim(L^{F_1})=s^2-s+1.
\]
\par\medskip\noindent
For every $s\ge3$, we prove
\[
  \CCdim(L^{F_1})
  \ge
  \bigl(s-\lfloor s/3\rfloor\bigr)
  \left(\left\lceil\sqrt{s\lfloor s/3\rfloor}\right\rceil-1\right)
  =\frac{2}{3\sqrt3}s^2+O(s).
\]
\par\noindent
Combining this with the affine-dimension upper bound yields
$\CCdim(L^{F_1})=\Theta(s^2)$.  Prior work established
$s^2+1$-coordinate affine or shifted low-rank representations and
quadratic-dimensional calibrated surrogates
\citep{nowak2019sharp,zhang2020convex}.  We are not aware of a prior exact-rank
formula for this matrix or an established \Fone{}-specific quadratic lower
bound on convex calibration dimension.

\section{Setup and background}

Fix an integer $s\ge1$.  Let $[s]=\{1,\ldots,s\}$ and let
$\cY=2^{[s]}$.  For $A\subseteq[s]$, let
$\chi_A\in\{0,1\}^s$ denote its incidence vector, so
$(\chi_A)_j=\mathbf{1}\{j\in A\}$.  We identify $A$ with $\chi_A$ when
convenient.  We write $\one_d\in\R^d$ for the $d$-dimensional all-ones vector
and, when coordinates are indexed by a finite set $I$, write $\one_I\in\R^I$
for the corresponding all-ones vector.  Finally, $\mathbf{1}\{E\}$ denotes the
scalar indicator of a condition $E$.  For an indexed vector $x$, the notation
$x_S$ denotes its restriction to coordinates in $S$; for a matrix $M$,
$M_{S,T}$ denotes the corresponding submatrix, and a dot denotes all indices,
as in the column $M_{\cdot,j}$.  Complements are taken in the ambient index
set.  For $A,B\in\cY$, define the \Fone{} score by
\begin{equation}\label{eq:f1-score}
  F(A,B)=
  \begin{cases}
    1, & A=B=\varnothing,\\[1mm]
    \displaystyle\frac{2|A\cap B|}{|A|+|B|},
      & |A|+|B|>0.
  \end{cases}
\end{equation}
We call $A\in\cY$ an \emph{outcome} (the true label set) and
$B\in\cY$ a \emph{report} (the predicted label set).  Thus ``report'' and
``prediction'' are synonymous in this paper.  The corresponding loss is
\begin{equation}
  \ell(A,B)=1-F(A,B).
\end{equation}
Throughout, we study the decision-theoretic, example-based loss obtained by
applying \Fone{} separately to each pair of true and predicted label sets and
then taking its expectation under an arbitrary conditional distribution on
$\cY$.  This differs from objectives that aggregate confusion counts over
examples or labels before applying \Fone{}, including common micro- and
macro-averaged formulations; their Bayes analyses are generally incomparable
\citep{ye2012optimizing,koyejo2015consistent}.

Let $N=2^s$, let $F=(F(A,B))_{A,B\in\cY}\in\R^{N\times N}$ be the score
matrix, and let $J=\one_N\one_N^\top$.  The \Fone{} loss matrix is
$L^{F_1}:=(\ell(A,B))_{A,B\in\cY}=J-F$; throughout, we abbreviate it as
$L:=L^{F_1}$.  Thus
$L-J=-F$.

For a finite family $S\subseteq\R^r$, let $\aff(S)$ denote its affine hull and
define $\affdim(S)=\dim\aff(S)$.  For a matrix
$M=[m_1\ \cdots\ m_k]$, its column-affine dimension is
\[
  \affdim(M):=\affdim\{m_1,\ldots,m_k\}.
\]

Let $\Delta_N=\{p\in\R_+^N:\one_N^\top p=1\}$.
For $p\in\Delta_N$, its coordinates are indexed by $A\in\cY$; define
$\operatorname{supp}(p):=\{A\in\cY:p_A>0\}$ and
$\|p\|_0:=|\operatorname{supp}(p)|$.
The conditional risk of report $B\in\cY$ is $p^\top L_{\cdot,B}$.  Its trigger
probability set is
\begin{equation}\label{eq:trigger-set}
  Q_B^L
  =\left\{p\in\Delta_N:
     p^\top L_{\cdot,B}\le p^\top L_{\cdot,B'}
     \text{ for every }B'\in\cY\right\}.
\end{equation}
Write
$\operatorname{opt}_L(p)=\arg\min_{B\in\cY}p^\top L_{\cdot,B}$
for the set of Bayes-optimal reports under $p$.

For completeness, let $\mathcal C\subseteq\R^d$ be convex, let
$\psi:\mathcal C\to\R_+^N$ have convex coordinate functions, and let
$\operatorname{pred}:\mathcal C\to\cY$ be a link.  The pair
$(\psi,\operatorname{pred})$ is \emph{$L$-calibrated} if, for every
$p\in\Delta_N$,
\begin{equation}\label{eq:calibration-definition}
  \inf_{\substack{u\in\mathcal C:\\
        \operatorname{pred}(u)\notin\operatorname{opt}_L(p)}}
       p^\top\psi(u)
  >
  \inf_{u\in\mathcal C}p^\top\psi(u).
\end{equation}
The convex calibration dimension $\CCdim(L)$ is the smallest $d$ for which
such a convex surrogate and link exist
\citep[Definitions~1 and~10]{ramaswamy2016ccdim}.  We use two general results
from that work:
\begin{equation}\label{eq:ccdim-upper}
  \CCdim(L)\le\affdim(L),
\end{equation}
and, for any $B\in\cY$ and $p\in Q_B^L$,
\begin{equation}\label{eq:ccdim-trigger-lower}
  \CCdim(L)
  \ge \|p\|_0-\mu_{Q_B^L}(p)-1.
\end{equation}
Here, for $Q\subseteq\Delta_N$ and $p\in Q$,
\[
  \operatorname{dir}_Q(p)=\left\{v\in\R^N:
    \begin{array}{l}
      \text{there exists $\varepsilon_0>0$ such that }p+\varepsilon v\in Q\\
      \text{for every $\varepsilon\in(0,\varepsilon_0)$}
    \end{array}
  \right\}
\]
is the cone of feasible directions of $Q$ at $p$, and
\[
  \operatorname{lin}_Q(p)
  :=\operatorname{dir}_Q(p)\cap\bigl(-\operatorname{dir}_Q(p)\bigr),
  \qquad
  \mu_Q(p)=\dim\operatorname{lin}_Q(p).
\]
The upper bound \eqref{eq:ccdim-upper} and lower bound
\eqref{eq:ccdim-trigger-lower} are
Theorems~12 and~16, respectively, of \citet{ramaswamy2016ccdim}.

\section{Related work}

\paragraph{Bayes-optimal prediction and plug-in methods.}
\citet{dembczynski2011exact} showed that Bayes-optimal multi-label \Fone{}
prediction can be computed using $s^2+1$ quantities from the conditional label
distribution, rather than all $2^s$ probabilities.  The expanded analysis of
\citet{waegeman2014bayes} established worst-case regret results for several
approximate \Fone{} maximizers and presented the general F-measure maximizer in a
decision-theoretic framework.  \citet{dembczynski2013optimizing} developed a
statistically consistent plug-in algorithm that estimates the required
quadratic collection of quantities.  More recently,
\citet{wang2026revisiting} proposed sampling-based estimation of these
quantities together with faster inference.  These results establish that a
quadratic number of statistics is sufficient; they do not assert its
necessity or calculate the rank of the \Fone{} loss matrix.

\paragraph{Low-rank calibrated surrogates.}
\citet{nowak2019sharp} gave $s^2+1$-coordinate affine decompositions of the
multi-label F-score loss and a consistent quadratic least-squares surrogate,
together with calibration and statistical analyses.  Their supplementary
analysis also noted that the full-support, all-reports-tied sufficient
condition for the standard affine-dimension lower bound fails even for $s=2$,
leaving optimality of the surrogate dimension unresolved.  General constructions of
calibrated convex surrogates for low-rank loss matrices were developed by
\citet{ramaswamy2013lowrank}, while
\citet{ramaswamy2014output} established consistency guarantees for output-code
reductions, including probabilistic code matrices.  Applying these tools to
multi-label F-measure, \citet{zhang2020convex} proved
$\rank(L^{F_\beta}-J)\le s^2+1$, constructed $s^2+1$-dimensional convex
calibrated surrogates, and gave a regret-transfer bound.  Our rank theorem
sharpens their bound for $\beta=1$ and identifies the exact affine dimension.

\paragraph{Calibration dimension and elicitation complexity.}
\citet{ramaswamy2012classification} introduced classification calibration
dimension for general multiclass losses; the expanded treatment of
\citet{ramaswamy2016ccdim} developed convex calibration dimension and derived
both affine-dimension upper bounds and feasible-subspace lower bounds, with
applications to subset-ranking losses.
Closely related lower-bound frameworks arise from property elicitation
\citep{agarwal2015property,frongillo2015elicitation} and from the $d$-flat
approach of \citet{finocchiaro2021lower}.  Polyhedral surrogate dimension has
also been studied through embeddings \citep{finocchiaro2020embedding,
finocchiaro2024framework}.  The affine-dimension lower bound of
\citet{agarwal2015property} applies to
calibrated linear properties and therefore does not lower-bound arbitrary
convex calibrated surrogates, while polyhedral embedding results do not
automatically lower-bound all convex calibrated surrogates.  More generally,
ordinary matrix rank cannot by itself give such a lower bound: a nonlinear
surrogate may preserve only the Bayes decision rather than reconstructing all
conditional risks.  Our second result instead exhibits an \Fone{}-specific
trigger point with quadratically many independent active directions.

\section{Exact rank of the \Fone{} loss matrix}\label{sec:rank}

We begin with two elementary linear-algebra facts.

\begin{lemma}[Rank of subset-incidence matrices]\label{lem:incidence-rank}
For $1\le k\le s$, let
$W_k\in\{0,1\}^{\binom{s}{k}\times s}$ be indexed by the $k$-subsets of
$[s]$ and the elements of $[s]$, with
$(W_k)_{A,j}=\mathbf{1}\{j\in A\}$.  Then
\[
  \rank(W_k)=
  \begin{cases}
    s, & 1\le k\le s-1,\\
    1, & k=s.
  \end{cases}
\]
\end{lemma}

\noindent The proof is given in Appendix~\ref{app:proof-incidence-rank}.

\begin{lemma}[A positive-definite Cauchy matrix]\label{lem:cauchy-pd}
The matrix $C\in\R^{s\times s}$ defined by
\[
  C_{kk'}=\frac{2}{k+k'},\qquad k,k'\in[s],
\]
is positive definite.
\end{lemma}

\noindent The proof is given in Appendix~\ref{app:proof-cauchy-pd}.

\begin{theorem}[Exact \Fone{} ranks and affine dimension]\label{thm:exact-rank}
For every $s\ge1$, under the convention in \eqref{eq:f1-score},
\[
  \rank(F)=\rank(L-J)=\rank(L)=s^2-s+2.
\]
Moreover,
\[
  \affdim(L)=s^2-s+1.
\]
\end{theorem}

\noindent\emph{Proof sketch.}
On the nonempty subsets, the score matrix factors as
$K=P(C\otimes I_s)P^\top$, where $\otimes$ denotes the Kronecker product and
grouping $P$ by set cardinality gives the subset-incidence blocks
$W_1,\ldots,W_s$.  Lemma~\ref{lem:incidence-rank} gives
$\rank(P)=s^2-s+1$, while Lemma~\ref{lem:cauchy-pd} implies
$C\otimes I_s\succ0$; the positive-middle rank identity therefore gives
$\rank(K)=\rank(P)$.  The isolated empty-set entry adds one to $\rank(F)$, and
$L-J=-F$ then gives the shifted rank.  For the ordinary loss rank, the block
form of $L$ together with $\one_{N-1}\in\range(K)$ yields
$\ker L=\{0\}\times\ker K$, hence $\rank(L)=1+\rank(K)$.  Finally,
$\one_N\in\range(F)$ and the symmetry of $F$ show that appending an
all-ones row does not increase rank, so $\affdim(F)=\rank(F)-1$; the
invertible affine map
$x\mapsto\one_N-x$ from score columns to loss columns then gives the stated
value of $\affdim(L)$.

\noindent The full proof is given in Appendix~\ref{app:proof-exact-rank}.

\begin{corollary}[Improved surrogate upper bound]\label{cor:upper}
For every $s\ge1$,
\[
  \CCdim(L^{F_1})\le s^2-s+1.
\]
\end{corollary}

\noindent The proof is given in Appendix~\ref{app:proof-upper}.

\begin{remark}[Empty-set convention]\label{rem:empty-convention}
If instead $F(\varnothing,\varnothing)=0$, then
\[
  \rank(F)=\rank(L-J)=s^2-s+1,
  \qquad
  \rank(L)=s^2-s+2,
\]
and $\affdim(L)=s^2-s+1$.  The verification is given in
Appendix~\ref{app:proof-empty-convention}.
\end{remark}

\section{A quadratic lower bound on convex calibration dimension}
\label{sec:ccdim-lower}

We now turn from linear-algebraic rank to the local geometry of Bayes-optimal
reports.  The construction has three ingredients.  First, a set $T$ of
``core'' labels is made more likely than the remaining labels.  Second, the
marginals on each outcome-cardinality layer are chosen so that all and only
the predictions containing $T$ tie.  Third, the active loss columns on the
support of the resulting distribution are shown to have quadratic affine
dimension.

\begin{theorem}[Quadratic CC-dimension lower bound]
\label{thm:ccdim-lower}
Let $s\ge3$ and define
\begin{equation}\label{eq:parameters}
  t=\left\lfloor\frac{s}{3}\right\rfloor,
  \qquad
  n=s-t,
  \qquad
  h=\left\lceil\sqrt{st}\right\rceil-1.
\end{equation}
Then
\begin{equation}\label{eq:ccdim-finite-lower}
  \CCdim(L^{F_1})\ge hn
  =
  \bigl(s-\lfloor s/3\rfloor\bigr)
  \left(\left\lceil\sqrt{s\lfloor s/3\rfloor}\right\rceil-1\right).
\end{equation}
Consequently,
\begin{equation}\label{eq:ccdim-asymptotic-lower}
  \CCdim(L^{F_1})
  \ge \frac{2}{3\sqrt3}s^2-O(s)
  =\Omega(s^2).
\end{equation}
For $s\ge12$, one has the explicit estimate
$\CCdim(L^{F_1})\ge2s^2/9$.
\end{theorem}

\noindent\emph{Proof sketch.}
Fix a core set $T\subseteq[s]$ of size $t$ and write $O=[s]\setminus T$.  Mix
full-support distributions on the outcome-cardinality layers $1,\ldots,h$,
with core- and optional-label inclusion marginals chosen so that all reports
$B_C=T\cup C$, $C\subseteq O$, tie, while every report missing a core label is
strictly worse.  This gives a witness $p$ supported on
$\cU=\{A\subseteq[s]:1\le |A|\le h\}$ whose Bayes-optimal reports are exactly
the $B_C$.  Restricted to $\cU$, a layerwise coefficient encoding and
Cauchy-matrix nonsingularity split the active loss differences into $h(n-1)$
optional-label swap directions and $h$ cardinality-change directions, so the
tied columns have affine dimension $hn$.  With $B_0=T$ and $R=|\cU|$, these active
constraints together with normalization leave a two-sided feasible subspace of
dimension $R-hn-1$.  Substitution into \eqref{eq:ccdim-trigger-lower} yields
$\CCdim(L^{F_1})\ge hn$; \eqref{eq:parameters} gives the remaining estimates.

\noindent The full proof is given in Appendix~\ref{app:proof-ccdim-lower}.

\begin{remark}[Why the rank proof is not a CC-dimension lower bound]
Theorem~\ref{thm:exact-rank} shows that any affine reconstruction of all \Fone{}
conditional risks needs $s^2-s+1$ coordinates.  Indeed, $L$ is symmetric, so
its row- and column-affine dimensions agree, while an affine image of
$\R^d$ has affine dimension at most $d$.  Calibration, however, only requires
recovery of a Bayes-optimal report.  A nonlinear convex surrogate may discard
directions that change risk values without changing their minimizer.
Theorem~\ref{thm:ccdim-lower} avoids that gap: its witness forces $hn$
independent loss differences to be simultaneously active at one trigger point,
and the feasible-subspace theorem applies to arbitrary convex calibrated
surrogates.
\end{remark}

\noindent The lower bound in Theorem~\ref{thm:ccdim-lower} is unchanged under
the alternative convention $F(\varnothing,\varnothing)=0$, because its witness
assigns no mass to the empty outcome.

\section{Conclusion}\label{sec:conclusion}

The rank calculation closes the gap left by the earlier $s^2+1$ upper bound:
under the convention in \eqref{eq:f1-score}, the shifted matrix has exact rank
$s^2-s+2$.  More importantly, the calibration-dimension result shows that the
quadratic dependence on $s$ is not an artifact of a particular output-code
construction.  Under the standard distribution-free notion of
$L$-calibration in \eqref{eq:calibration-definition}, every convex calibrated
surrogate for the $s$-label \Fone{} loss has prediction dimension
$\Omega(s^2)$.

The bounds do not determine the exact convex calibration dimension.  Our
results give
\[
  \frac{2}{3\sqrt3}s^2-O(s)
  \le \CCdim(L^{F_1})
  \le s^2-s+1.
\]
Closing the constant-factor gap requires either a harder trigger distribution,
possibly with several interacting core sets, or a surrogate construction that
uses less than the full affine dimension.  Another natural direction is to
extend the exact rank and calibration-dimension analysis to the asymmetric
$F_\beta$ family and to related set-similarity losses such as Jaccard loss.

\clearpage
\appendix

\section{Proofs for the exact-rank results}\label{app:rank-proofs}

\subsection{Proof of Lemma~\ref{lem:incidence-rank}}
\label{app:proof-incidence-rank}

\begin{proof}[Proof of Lemma~\ref{lem:incidence-rank}]
We calculate the null space of $W_k$.  For $x=(x_1,\ldots,x_s)^\top$,
the coordinate of $W_kx$ indexed by a $k$-subset $A$ is
\begin{equation}\label{eq:incidence-null-equation}
  (W_kx)_A=\sum_{q\in A}x_q.
\end{equation}

First suppose $1\le k\le s-1$ and let $x\in\ker(W_k)$.  Then
\eqref{eq:incidence-null-equation} is zero for every $k$-subset $A$.
Fix two distinct labels $i,j\in[s]$.  Since $k\le s-1$, one has
$k-1\le s-2$, so there exists
\[
  S\subseteq[s]\setminus\{i,j\},
  \qquad |S|=k-1;
\]
when $k=1$, take $S=\varnothing$.  Both $S\cup\{i\}$ and
$S\cup\{j\}$ are $k$-subsets.  Applying
\eqref{eq:incidence-null-equation} to these two sets gives
\[
  \sum_{q\in S}x_q+x_i=0,
  \qquad
  \sum_{q\in S}x_q+x_j=0.
\]
Subtracting the equations yields $x_i=x_j$.  Because $i$ and $j$ were
arbitrary, all coordinates of $x$ have a common value $c$.  Substituting
into \eqref{eq:incidence-null-equation} for any $k$-subset gives
\[
  0=\sum_{q\in A}x_q=kc.
\]
Over $\R$, and since $k\ge1$, this implies $c=0$.  Hence
$\ker(W_k)=\{0\}$, so all $s$ columns of $W_k$ are linearly independent
and $\rank(W_k)=s$.

Finally, if $k=s$, there is only one $s$-subset, namely $[s]$, and
\[
  W_s=\begin{pmatrix}1&1&\cdots&1\end{pmatrix}.
\]
This is a single nonzero row, so $\rank(W_s)=1$.
\end{proof}

\subsection{Proof of Lemma~\ref{lem:cauchy-pd}}
\label{app:proof-cauchy-pd}

\begin{proof}[Proof of Lemma~\ref{lem:cauchy-pd}]
The matrix is symmetric, so it remains to show that its quadratic form is
strictly positive away from the origin.  The elementary identity
\begin{equation}\label{eq:cauchy-entry-integral}
  \frac{2}{k+k'}
  =2\int_0^1 \tau^{k+k'-1}\,d\tau
\end{equation}
gives, for $v=(v_1,\ldots,v_s)^\top\in\R^s$,
\begin{align}
  v^\top Cv
  &=\sum_{k=1}^s\sum_{k'=1}^s
    v_kv_{k'}\frac{2}{k+k'} \notag\\
  &=2\int_0^1
    \sum_{k=1}^s\sum_{k'=1}^s
    v_kv_{k'} \tau^{k+k'-1}\,d\tau \notag\\
  &=2\int_0^1 \tau^{-1}
    \left(\sum_{k=1}^s v_k\tau^k\right)^2\,d\tau.
    \label{eq:cauchy-quadratic-integral}
\end{align}
Define
\[
  q_v(\tau)=\sum_{k=1}^s v_k\tau^k.
\]
Because $q_v$ has no constant term, it can be written as
$q_v(\tau)=\tau\widetilde q_v(\tau)$ for a polynomial
$\widetilde q_v$.  Consequently,
\[
  \frac{q_v(\tau)^2}{\tau}=\tau\widetilde q_v(\tau)^2,
\]
which extends continuously to $\tau=0$.  This verifies that the integral in
\eqref{eq:cauchy-quadratic-integral} is finite.

The integrand in \eqref{eq:cauchy-quadratic-integral} is nonnegative.  If
$v\ne0$, then $q_v$ is a nonzero polynomial and therefore cannot vanish on
the entire interval $(0,1)$.  Thus there is some $\tau_0\in(0,1)$ for which
$q_v(\tau_0)\ne0$.  By continuity, $q_v(\tau)^2/\tau$ is strictly positive on a
neighborhood of $\tau_0$, and hence
\[
  v^\top Cv
  =2\int_0^1\frac{q_v(\tau)^2}{\tau}\,d\tau>0.
\]
This proves that $C$ is positive definite.

Equivalently, $C$ is the Gram matrix in $L^2(0,1)$ of the linearly
independent functions
\[
  \phi_k(\tau)=\sqrt{2}\,\tau^{k-1/2},
  \qquad k\in[s],
\]
because
\[
  \langle\phi_k,\phi_{k'}\rangle_{L^2(0,1)}
  =2\int_0^1\tau^{k+k'-1}\,d\tau
  =\frac{2}{k+k'}=C_{kk'}.
\]
\end{proof}

\subsection{A positive-middle rank identity}
\label{app:proof-positive-middle}

The proof of Theorem~\ref{thm:exact-rank} uses the following identity.  If
$M\succ0$ and the matrix products below are defined, then
\begin{equation}\label{eq:positive-middle}
  \rank(XMX^\top)=\rank(X),
  \qquad
  \range(XMX^\top)=\range(X).
\end{equation}

\begin{proof}[Proof of the identity]
Write $M=GG^\top$ with $G$ invertible.  Then
$XMX^\top=(XG)(XG)^\top$.  For every matrix $Y$,
$\range(YY^\top)=\range(Y)$.  Since right multiplication by the invertible
matrix $G$ preserves both rank and column space, the two conclusions follow.
\end{proof}

\subsection{Proof of Theorem~\ref{thm:exact-rank}}
\label{app:proof-exact-rank}

\begin{proof}[Proof of Theorem~\ref{thm:exact-rank}]
Write $\rho=s^2-s+1$.  The proof has four steps.  We first factor the score
matrix on the nonempty sets and show that its rank is $\rho$.  We then add the
isolated empty-set entry, analyze the ordinary loss matrix through its null
space, and finally compute the affine dimension of the columns.

\paragraph{Step 1: Factor the nonempty score block.}
Let $\cY_+=\cY\setminus\{\varnothing\}$ and order $\varnothing$ first.
Because the empty set has score zero against every nonempty set, while
$F(\varnothing,\varnothing)=1$,
\begin{equation}\label{eq:F-block}
  F=\begin{pmatrix}1&0\\0&K\end{pmatrix},
\end{equation}
where $K$ is the score matrix indexed by $\cY_+$.  Explicitly,
\[
  K_{A,B}=F(A,B)=\frac{2|A\cap B|}{|A|+|B|},
  \qquad A,B\in\cY_+.
\]

To factor $K$, define $P\in\R^{(N-1)\times s^2}$, with columns indexed by
$(k,j)\in[s]\times[s]$, by
\[
  P_{A,(k,j)}
  =\mathbf{1}\{|A|=k\}\mathbf{1}\{j\in A\},
  \qquad A\in\cY_+.
\]
Thus a row of $P$ records both the size of a set and the labels it contains.
Here $\otimes$ denotes the Kronecker product, and $I_s$ is the
$s\times s$ identity matrix.  Since
\[
  (C\otimes I_s)_{(k,j),(k',q)}
  =\frac{2}{k+k'}\mathbf{1}\{j=q\},
\]
the identity factor matches the same label on the two sides.  Consequently,
for nonempty $A,B$,
\begin{align*}
  \left[P(C\otimes I_s)P^\top\right]_{A,B}
  &=\sum_{j\in A\cap B}\frac{2}{|A|+|B|}\\
  &=\frac{2|A\cap B|}{|A|+|B|}=K_{A,B}.
\end{align*}
Hence
\begin{equation}\label{eq:score-factorization}
  K=P(C\otimes I_s)P^\top.
\end{equation}

\paragraph{Step 2: Compute the score ranks.}
Group the rows and columns of $P$ by the cardinality $k$.  Rows indexed by
$k$-subsets are nonzero only in columns $(k,j)$, so $P$ is block diagonal
with blocks $W_1,\ldots,W_s$.  Lemma~\ref{lem:incidence-rank} therefore gives
rank $s$ to each of the first $s-1$ blocks and rank $1$ to the final block.
Consequently,
\[
  \rank(P)
  =\sum_{k=1}^s\rank(W_k)
  =s(s-1)+1
  =s^2-s+1=\rho.
\]
Lemma~\ref{lem:cauchy-pd} gives $C\succ0$, and hence
$C\otimes I_s\succ0$.  Applying \eqref{eq:positive-middle} to
\eqref{eq:score-factorization} yields
\begin{equation}\label{eq:K-rank-range}
  \rank(K)=\rank(P)=\rho,
  \qquad
  \range(K)=\range(P).
\end{equation}
The scalar block $1$ in \eqref{eq:F-block} is disjoint from $K$, so it adds
one to the rank:
\[
  \rank(F)=1+\rank(K)=\rho+1=s^2-s+2.
\]
Finally, $L=J-F$, so $L-J=-F$ and
\[
  \rank(L-J)=\rank(F)=s^2-s+2.
\]

\paragraph{Step 3: Compute the ordinary rank of the loss matrix.}
The preceding identity does not yet determine $\rank(L)$, because adding the
rank-one matrix $J$ can change rank.  We instead calculate $\ker L$.

Let $\one_+:=\one_{N-1}$ and define $\omega\in\R^{s^2}$ by
$\omega_{(k,j)}=1/k$.  If $A$ is nonempty and $|A|=k$, then
\[
  (P\omega)_A=\sum_{j\in A}\frac1k=1.
\]
Thus $P\omega=\one_+$, so \eqref{eq:K-rank-range} implies
$\one_+\in\range(K)$.  Choose $u$ satisfying
\begin{equation}\label{eq:Ku-oneplus}
  Ku=\one_+.
\end{equation}
Subtracting the block matrix \eqref{eq:F-block} from $J$ gives
\begin{equation}\label{eq:L-block}
  L=
  \begin{pmatrix}
    0&\one_+^\top\\
    \one_+&\one_+\one_+^\top-K
  \end{pmatrix}.
\end{equation}

Take $(\alpha,x)\in\R\times\R^{N-1}$.  From
$L\binom{\alpha}{x}=0$, the first block row gives $\one_+^\top x=0$.  Using
this in the second block row gives
\[
  \alpha\one_++\one_+(\one_+^\top x)-Kx=0
  \quad\Longrightarrow\quad
  Kx=\alpha\one_+.
\]
Together with \eqref{eq:Ku-oneplus}, this means
$K(x-\alpha u)=0$, so
\[
  x=\alpha u+z
  \qquad\text{for some }z\in\ker K.
\]
Because $K$ is symmetric, $\one_+=Ku$ and $Kz=0$ imply
\[
  \one_+^\top z=u^\top Kz=0.
\]
Therefore the remaining equation $\one_+^\top x=0$ becomes
\begin{equation}\label{eq:alpha-oneplus-u}
  0=\one_+^\top x=\alpha\one_+^\top u.
\end{equation}
The factorization \eqref{eq:score-factorization} shows that $K\succeq0$, and
\[
  \one_+^\top u=u^\top Ku>0.
\]
Indeed, if $u^\top Ku=0$ for a positive-semidefinite $K$, then $Ku=0$,
contradicting \eqref{eq:Ku-oneplus}.  Equation
\eqref{eq:alpha-oneplus-u} now forces
$\alpha=0$, after which $x=z\in\ker K$.  Conversely, every
$(0,z)$ with $z\in\ker K$ lies in $\ker L$.  We have proved
\[
  \ker L=\{0\}\times\ker K.
\]
Thus $L$ has exactly the same nullity as $K$, but it has one additional row
and column.  Hence
\[
  \rank(L)=1+\rank(K)=\rho+1=s^2-s+2.
\]

\paragraph{Step 4: Compute the affine dimension.}
For any matrix $M$ with $m$ columns, augmenting each column with a final
coordinate equal to one converts affine dependence into linear dependence.
Therefore
\begin{equation}\label{eq:affine-rank}
  \affdim(M)
  =\rank\begin{pmatrix}M\\\one_m^\top\end{pmatrix}-1.
\end{equation}
Equation \eqref{eq:Ku-oneplus} and the block form \eqref{eq:F-block} show that
\[
  F\binom{1}{u}=\binom{1}{Ku}=\binom{1}{\one_+}=\one_N.
\]
Thus $\one_N\in\range(F)$.  Since $F$ is symmetric,
$\one_N^\top$ already belongs to the row space of $F$.  Appending that row
therefore does not increase rank, and \eqref{eq:affine-rank} gives
\[
  \affdim(F)=\rank(F)-1=s^2-s+1.
\]
\vspace{0.5\baselineskip}\par\noindent
Finally, the invertible affine map $x\mapsto\one_N-x$ sends every score
column to the corresponding loss column.  Invertible affine maps preserve
affine dimension, so
\[
  \affdim(L)=\affdim(F)=s^2-s+1.
\]
\end{proof}

\subsection{Proof of Corollary~\ref{cor:upper}}
\label{app:proof-upper}

\begin{proof}[Proof of Corollary~\ref{cor:upper}]
The CC-dimension bound follows from \eqref{eq:ccdim-upper} and
Theorem~\ref{thm:exact-rank}.
\end{proof}

\subsection{Verification of Remark~\ref{rem:empty-convention}}
\label{app:proof-empty-convention}

\begin{proof}[Verification of Remark~\ref{rem:empty-convention}]
Let $K$ be the nonempty score block from \eqref{eq:F-block}, and write
$\one_+=\one_{N-1}$ again.  Under the alternative convention,
\[
  F=\begin{pmatrix}0&0\\0&K\end{pmatrix},
  \qquad
  L=\begin{pmatrix}
      1&\one_+^\top\\
      \one_+&\one_+\one_+^\top-K
    \end{pmatrix}.
\]
Thus $\rank(F)=\rank(K)=s^2-s+1$, and $L-J=-F$ has the same rank.

Equation~\eqref{eq:Ku-oneplus} gives $\one_+\in\range(K)$.  If
$L\binom{\alpha}{x}=0$, the first block row gives
$\alpha+\one_+^\top x=0$, while the second gives
\[
  (\alpha+\one_+^\top x)\one_+-Kx=0.
\]
Hence $Kx=0$.  Since $K$ is symmetric and $\one_+\in\range(K)$, one has
$\one_+^\top x=0$, and therefore $\alpha=0$.  Conversely, every
$(0,x)$ with $x\in\ker K$ lies in $\ker L$.  Consequently,
\[
  \ker L=\{0\}\times\ker K,
  \qquad
  \rank(L)=1+\rank(K)=s^2-s+2.
\]
\vspace{0.5\baselineskip}\par\noindent
Finally, the zero score column belongs to the family of columns of $F$, so
their affine hull has dimension equal to their linear span:
$\affdim(F)=\rank(F)=s^2-s+1$.  The invertible affine map
$x\mapsto\one_N-x$ sends the score columns to the loss columns, and hence
$\affdim(L)=s^2-s+1$ as well.
\end{proof}

\clearpage
\section{Proof of the quadratic CC-dimension lower bound}
\label{app:proof-ccdim-lower}

\begin{proof}[Proof of Theorem~\ref{thm:ccdim-lower}]
By \eqref{eq:ccdim-trigger-lower}, it is enough to construct a distribution
$p$ and a Bayes-optimal prediction $B_0$ such that, with
$R=|\operatorname{supp}(p)|$,
\[
  \mu_{Q_{B_0}^L}(p)=R-hn-1.
\]
Indeed, substituting this identity into \eqref{eq:ccdim-trigger-lower} gives
the desired lower bound $hn$.  The proof has five steps:
\begin{enumerate}
  \item construct $p$ with full support on outcome-cardinality layers
        $1,\ldots,h$;
  \item identify exactly which predictions are Bayes optimal under $p$;
  \item show that their loss columns, restricted to the support of $p$, have
        affine dimension $hn$;
  \item translate that affine dimension into the required feasible-subspace
        dimension; and
  \item simplify the resulting finite bound asymptotically and for
        $s\ge12$.
\end{enumerate}

Fix $T\subseteq[s]$ with $|T|=t$, and write
$O=[s]\setminus T$, so $|O|=n$.  We call $T$ the set of core labels and $O$
the set of optional labels.

\paragraph{Preliminary parameter bounds.}
Since $t=\lfloor s/3\rfloor$, one has $3t\le s\le3t+2$ and
$t\ge1$.  Hence $n=s-t\ge2t$.  Therefore
\[
  n^2-st=n^2-t(n+t)\ge t^2>0.
\]
Since $\lceil x\rceil-1<x$ for $x>0$, and since $st\ge3>1$ implies
$h=\lceil\sqrt{st}\rceil-1\ge1$,
\begin{equation}\label{eq:parameter-inequalities}
  1\le h<\sqrt{st}<n.
\end{equation}
In particular, $h\le n-1\le s-2$ and $h^2<st$.
The inequality $h^2<st$ will ensure that the desired inclusion marginals
lie strictly between zero and one, while $h\le n-1$ will provide enough
optional-set cardinalities for the dimension calculation.

\paragraph{Step 1: Construct a full-support witness distribution.}
The following marginals are chosen so that, in Step~2, every prediction
containing all core labels has the same conditional score.
For each $m\in[h]$, define
\begin{equation}\label{eq:alpha-beta}
  \alpha_m=\frac{m(m+t)}{t(m+s)},
  \qquad
  \beta_m=\frac{m}{m+s}.
\end{equation}
Let $x^{(m)}\in\R^s$ have coordinate $\alpha_m$ on $T$ and coordinate
$\beta_m$ on $O$.  Its coordinate sum is
\[
  t\alpha_m+n\beta_m
  =\frac{m(m+t)+nm}{m+s}=m.
\]
Moreover, both $\alpha_m$ and $\beta_m$ are positive,
$\beta_m<1$, and
\[
  \alpha_m<1
  \quad\Longleftrightarrow\quad
  m^2<st,
\]
which follows from $m\le h<\sqrt{st}$.  Hence $x^{(m)}$ lies in the
relative interior of the hypersimplex
\begin{equation}\label{eq:hypersimplex}
  \Delta(s,m)
  =\left\{x\in[0,1]^s:\sum_{i=1}^s x_i=m\right\}
  =\operatorname{conv}\{\chi_A:A\subseteq[s],\ |A|=m\}.
\end{equation}

Every relative-interior point of a finite polytope admits a convex
representation assigning positive weight to every vertex.  To see this,
let $\mathcal P$ be such a polytope and let $z$ be the uniform average of
its vertices.  For $x_\star\in\relint(\mathcal P)$, choose $\delta>0$
small enough that
$y=x_\star+\delta(x_\star-z)\in\mathcal P$.  Then
\[
  x_\star=\frac{1}{1+\delta}y+\frac{\delta}{1+\delta}z.
\]
Expanding $y$ as a convex combination of the vertices proves the claim.
It follows that, for each $m\in[h]$, writing
$\cY_m=\{A\subseteq[s]:|A|=m\}$, there are strictly positive weights
$\pi_m(A)$, $A\in\cY_m$, satisfying
\[
  x^{(m)}=\sum_{A\in\cY_m}\pi_m(A)\chi_A,
  \qquad \pi_m(A)>0,
  \qquad
  \sum_{A\in\cY_m}\pi_m(A)=1,
\]
so $\pi_m$ is a full-support distribution on the $m$-subsets.  Taking
coordinate $i$ gives
\begin{equation}\label{eq:layer-marginals}
  \sum_{\substack{A\in\cY_m\\i\in A}}\pi_m(A)
  =x_i^{(m)}
  =
  \begin{cases}
    \alpha_m,&i\in T,\\
    \beta_m,&i\in O.
  \end{cases}
\end{equation}
These are inclusion marginals, not an independence assumption.

Now draw $M$ uniformly from $[h]$ and, given $M=m$, draw the random outcome
$A\in\cY_m$ according to $\pi_m$.  Let $p$ be the marginal distribution of
$A$.  Its
support is exactly
\begin{equation}\label{eq:trigger-support}
  \cU=\{A\subseteq[s]:1\le |A|\le h\}.
\end{equation}
Indeed, every $m\in[h]$ has positive mixing probability and every $m$-subset
has positive probability under $\pi_m$, whereas no other cardinality layer is
sampled.  Thus $p_A>0$ exactly for $A\in\cU$.

\paragraph{Step 2: Identify the Bayes-optimal predictions.}
Fix a prediction $B\subseteq[s]$ and write
\[
  a=|B\cap T|,
  \qquad c=|B\cap O|,
  \qquad b=|B|=a+c.
\]
Thus $a$ and $c$ count the core and optional labels in $B$, respectively.
Conditional on $M=m$, Equation~\eqref{eq:layer-marginals} and linearity of
expectation give
\[
  \mathbb E[|A\cap B|\mid M=m]
  =\sum_{i\in B}\Pr(i\in A\mid M=m)
  =a\alpha_m+c\beta_m.
\]
Moreover, $|A|=m$ on this conditional event, so the denominator of the
\Fone{} score is the constant $m+b$.  Consequently,
\begin{align}
  \mathbb E[F(A,B)\mid M=m]
  &=\frac{2(a\alpha_m+c\beta_m)}{m+b} \notag\\
  &=\frac{2m}{m+s}\,
    \frac{b+(a/t)m}{m+b}.                    \label{eq:conditional-score}
\end{align}
The last equality uses
\[
  a\alpha_m+c\beta_m
  =\frac{m}{m+s}\left(\frac{a(m+t)}t+c\right)
  =\frac{m}{m+s}\left(b+\frac at m\right).
\]

Since $B\cap T\subseteq T$, one has $0\le a\le t$.  We now consider the two
possible cases.

\emph{All core labels are included: $a=t$.}
In this case $a=t$ is equivalent to $T\subseteq B$.  Equation
\eqref{eq:conditional-score} becomes
\begin{equation}\label{eq:active-score}
  \mathbb E[F(A,B)\mid M=m]=\frac{2m}{m+s},
\end{equation}
because the second fraction in \eqref{eq:conditional-score} is
$(b+m)/(m+b)=1$.  Hence every prediction containing $T$ has the same score,
regardless of which optional labels it contains.

\emph{At least one core label is missing: $a<t$.}
Subtracting the score in \eqref{eq:conditional-score} from the common score
in \eqref{eq:active-score} gives
\begin{equation}\label{eq:strict-score-deficit}
  \frac{2m}{m+s}
  -\mathbb E[F(A,B)\mid M=m]
  =\frac{2m^2(t-a)}{t(m+s)(m+b)}>0.
\end{equation}
Thus a prediction missing a core label has a strictly smaller expected score
on every layer $m\in[h]$.

Because the loss equals $1-F$, minimizing expected loss is equivalent to
maximizing expected score.  Averaging over $M$ preserves both the tie among
predictions containing $T$ and the strict gap for every other prediction.
Finally, every prediction containing $T$ has the unique representation
$T\cup C$, where $C=B\cap O\subseteq O$.  Therefore the Bayes-optimal
predictions under $p$ are exactly
\begin{equation}\label{eq:active-reports}
  B_C=T\cup C,
  \qquad C\subseteq O.
\end{equation}
In particular, let $B_0:=B_\varnothing=T$.  Then
$p\in Q_{B_0}^L$.

\paragraph{Step 3: Compute the affine dimension of the tied columns.}
For $C\subseteq O$, put $c=|C|$ and let
\[
  g_C=\bigl(F(A,B_C)\bigr)_{A\in\cU}.
\]
Because $L_{\cU,B_C}=\one_{\cU}-g_C$, it suffices to calculate the affine
dimension of the score vectors $g_C$.  The dimension count has a simple
interpretation.  There are $h$ outcome-cardinality layers.  On each layer,
changing $C$ produces $n-1$ directions that exchange optional labels while
keeping $|C|$ fixed, plus one direction that changes $|C|$.  We will make
this decomposition rigorous and obtain
\[
  h(n-1)+h=hn.
\]

\paragraph{Step 3a: Encode the score columns by coefficient matrices.}
For $c=0,\ldots,n$, set
\begin{equation}\label{eq:cauchy-vectors}
  \gamma_c=\left(\frac{1}{m+t+c}\right)_{m=1}^h\in\R^h.
\end{equation}
We use an $h\times(n+1)$ matrix to store coefficients.  Its rows are indexed
by the outcome size $m\in[h]$.  Column $0$ stores the common coefficient for
all core labels, and the remaining columns are indexed by the optional labels
$j\in O$.

Let $e_0,\widehat{\chi}_C\in\R^{n+1}$, where $e_0$ is one in column $0$ and
zero elsewhere, while $\widehat{\chi}_C$ is zero in column $0$ and is the
incidence vector of $C$ on the optional columns.  Define the coefficient matrix
\begin{equation}\label{eq:coefficient-matrix}
  V_C=\gamma_c(e_0+\widehat{\chi}_C)^\top\in\R^{h\times(n+1)}.
\end{equation}
This is an outer product.  Explicitly, if $c=|C|$, then
\[
  (V_C)_{m,0}=\frac1{m+t+c},
  \qquad
  (V_C)_{m,j}=\frac{\mathbf{1}\{j\in C\}}{m+t+c}
  \quad(j\in O).
\]

Define the linear map $\Psi:\R^{h\times(n+1)}\to\R^{\cU}$ by
\begin{equation}\label{eq:Psi-map}
  [\Psi(Z)](A)
  =Z_{m,0}|A\cap T|
   +\sum_{j\in O}Z_{m,j}\mathbf{1}\{j\in A\},
  \qquad m=|A|.
\end{equation}
Thus $\Psi$ converts a coefficient matrix into a vector indexed by the
outcomes in $\cU$.  For each $A\in\cU$, the coordinate $[\Psi(Z)](A)$ is
computed using row $|A|$ of $Z$.
If $A\in\cU$ and $m=|A|$, then
\[
  [\Psi(V_C)](A)
  =\frac{|A\cap T|+|A\cap C|}{m+t+c}
  =\frac{|A\cap B_C|}{|A|+|B_C|}
  =\frac12 F(A,B_C).
\]
Therefore
\[
  g_C=2\Psi(V_C).
\]

\paragraph{Step 3b: Check that the encoding loses no dimensions.}
To prove that $\Psi$ loses no dimensions, it is enough to prove that its
kernel is trivial.  Suppose, therefore, that $\Psi(Z)=0$.  We will show one
row of $Z$ at a time to be zero.

Fix $m\in[h]$.  Turn row $m$ into one coefficient $\zeta_j$ for each
label $j\in[s]$ by setting
\[
  \zeta_j=
  \begin{cases}
    Z_{m,0}, & j\in T,\\
    Z_{m,j}, & j\in O.
  \end{cases}
\]
Thus all core labels receive the common coefficient stored in column $0$,
while each optional label receives the coefficient in its own column.
Every $m$-subset $A$ belongs to $\cU$ because $m\in[h]$.  Therefore
Equation~\eqref{eq:Psi-map} now has the simpler form
\begin{equation}\label{eq:m-subset-zero-sum}
  \sum_{j\in A}\zeta_j=0
\end{equation}
for every $A\subseteq[s]$ with $|A|=m$.

We first show that any two label coefficients are equal.  Fix distinct labels
$i,j\in[s]$.  After removing $i$ and $j$, there are $s-2$ labels left.
Because $m\le h\le s-2$, we may choose
\[
  A_{ij}^{\circ}\subseteq[s]\setminus\{i,j\},
  \qquad |A_{ij}^{\circ}|=m-1;
\]
when $m=1$, take $A_{ij}^{\circ}=\varnothing$.  The sets
$A_{ij}^{\circ}\cup\{i\}$ and $A_{ij}^{\circ}\cup\{j\}$ are both
$m$-subsets, so
\eqref{eq:m-subset-zero-sum} gives
\[
  \sum_{k\in A_{ij}^{\circ}}\zeta_k+\zeta_i=0,
  \qquad
  \sum_{k\in A_{ij}^{\circ}}\zeta_k+\zeta_j=0.
\]
Subtracting these two equations cancels the common sum over
$A_{ij}^{\circ}$ and yields $\zeta_i-\zeta_j=0$.  Hence
$\zeta_i=\zeta_j$.  Since $i$ and $j$ were arbitrary, all label
coefficients have one common value, say
$\zeta_j=\zeta_\star$ for every $j\in[s]$.

Finally, apply \eqref{eq:m-subset-zero-sum} to any $m$-subset $A$.  It gives
\[
  0=\sum_{j\in A}\zeta_j=m\zeta_\star.
\]
Since $m\ge1$, we have $\zeta_\star=0$.  Therefore every coefficient assigned
from row $m$ is zero.  In particular, $Z_{m,j}=0$ for every $j\in O$; and,
because $T$ is nonempty ($t\ge1$), the coefficient of a core label gives
$Z_{m,0}=0$.  Thus row $m$ is zero.  The choice of $m\in[h]$ was arbitrary,
so every row of $Z$ is zero.  Consequently $Z=0$, the kernel of $\Psi$ is
trivial, and $\Psi$ is injective.

An injective linear map preserves the dimension of every subspace.  Define
the coefficient-matrix difference space
\begin{equation}\label{eq:D-definition}
  D=\operatorname{span}\{V_C-V_\varnothing:C\subseteq O\}.
\end{equation}
Since $g_C-g_\varnothing=2\Psi(V_C-V_\varnothing)$, we have
\[
  \operatorname{span}\{g_C-g_\varnothing:C\subseteq O\}=2\Psi(D),
  \qquad
  \dim\Psi(D)=\dim D.
\]
Thus it remains only to compute $\dim D$.

\paragraph{Step 3c: Count the label-swap directions.}
This step isolates changes that keep $|C|$ fixed and exchange only the
identities of the optional labels.  The argument has two parts.  First, at
each fixed cardinality $c$, these exchanges generate every zero-sum change
among the optional columns, with outcome-layer profile $\gamma_c$.  Second,
the profiles $\gamma_1,\ldots,\gamma_h$ form a basis of $\R^h$, so the
zero-sum change can
be chosen independently on each of the $h$ outcome-size layers.

\par\smallskip\noindent\emph{The one-layer swap space.}
The linear space of optional-coordinate changes that preserve their total
sum is
\[
  W_{\mathrm{swap}}=\{w\in\R^O:\one_O^\top w=0\},
  \qquad \dim W_{\mathrm{swap}}=n-1.
\]
Indeed, $W_{\mathrm{swap}}$ is the kernel of the nonzero summation map
$w\mapsto\one_O^\top w$ on $\R^O\cong\R^n$.

Across all $h$ layers, the full swap space is
\[
  \mathcal S
  =\left\{Z\in\R^{h\times(n+1)}:
      Z_{\cdot,0}=0_h,\quad
      Z_{\cdot,O}\one_O=0_h\right\}.
\]
The first condition makes the core column zero.  The second says that the
optional block of every row lies in $W_{\mathrm{swap}}$.  Hence $\mathcal S$
is naturally $h$ independent copies of $W_{\mathrm{swap}}$.  We now prove that
$\mathcal S\subseteq D$.

\emph{Fixed-cardinality swaps.}
Fix $c\in\{1,\ldots,n-1\}$.  If $C,C'\subseteq O$ both have size $c$,
then Equation~\eqref{eq:coefficient-matrix} shows that the common core term
$\gamma_c e_0^\top$ in $V_C$ and $V_{C'}$ cancels, giving
\begin{equation}\label{eq:fixed-cardinality-difference}
  V_C-V_{C'}
  =\gamma_c(\widehat{\chi}_C-\widehat{\chi}_{C'})^\top.
\end{equation}
Moreover, this matrix lies in $D$ because
\[
  V_C-V_{C'}
  =(V_C-V_\varnothing)-(V_{C'}-V_\varnothing),
\]
and the two matrices in parentheses are generators of $D$.

We next show that the optional-coordinate blocks of the incidence-vector
differences in \eqref{eq:fixed-cardinality-difference} generate all of
$W_{\mathrm{swap}}$.
Given
distinct $i,j\in O$, choose
\[
  C_0\subseteq O\setminus\{i,j\},
  \qquad |C_0|=c-1;
\]
this is possible because
$|O\setminus\{i,j\}|=n-2\ge c-1$ (and for $c=1$, take
$C_0=\varnothing$).  With $C=C_0\cup\{i\}$ and
$C'=C_0\cup\{j\}$, the common set $C_0$ cancels and
\[
  \widehat{\chi}_C-\widehat{\chi}_{C'}
  =\widehat{\chi}_{\{i\}}-\widehat{\chi}_{\{j\}}.
\]
Writing $e_i\in\R^O$ for the standard coordinate vector supported at label
$i$, the optional block of this difference is $e_i-e_j$.  These pairwise
coordinate differences span $W_{\mathrm{swap}}$.  Explicitly, after fixing
$j_\ast\in O$,
every $w\in W_{\mathrm{swap}}$ satisfies
\[
  w=\sum_{i\in O\setminus\{j_\ast\}}
       w_i(e_i-e_{j_\ast}),
\]
where the coefficient at $j_\ast$ is correct because
$w_{j_\ast}=-\sum_{i\ne j_\ast}w_i$.  Therefore
\eqref{eq:fixed-cardinality-difference} and linearity of $D$ imply
\begin{equation}\label{eq:fixed-c-swap-span}
  \widetilde w:=(0,w)\in\R^{n+1},
  \qquad \gamma_c\widetilde w^\top\in D
  \quad (w\in W_{\mathrm{swap}},\ 1\le c\le n-1).
\end{equation}
Equation~\eqref{eq:fixed-c-swap-span} defines the tilde explicitly:
the initial $0$ is the core coordinate, and $w$ occupies the $n$ optional
coordinates.  Thus $\gamma_c\widetilde w^\top$ has the same
$h\times(n+1)$ shape as the elements of $D$.

\emph{Combining the outcome-size layers.}
For a fixed $c$, Equation~\eqref{eq:fixed-c-swap-span} gives only matrices
whose variation across the $h$ rows follows the single profile $\gamma_c$.
To
combine these profiles, let
\[
  \Gamma=[\gamma_1\ \cdots\ \gamma_h]\in\R^{h\times h},
  \qquad
  \Gamma_{m,c}=\frac{1}{m+t+c}.
\]
The matrix $\Gamma$ is a Cauchy matrix: the row parameters $m+t$ are distinct,
the column parameters $c$ are distinct, and every denominator is nonzero.
More explicitly, the classical Cauchy determinant formula gives
\[
  \det \Gamma
  =
  \frac{
    \displaystyle\prod_{1\le m<m'\le h}(m'-m)
    \displaystyle\prod_{1\le c<c'\le h}(c'-c)
  }{
    \displaystyle\prod_{m=1}^h\prod_{c=1}^h(m+t+c)
  }.
\]
Every numerator factor is nonzero because the row and column parameters
are distinct, and every denominator factor is positive.  Thus
$\det \Gamma\ne0$, so $\Gamma$ is nonsingular and its columns
$\gamma_1,\ldots,\gamma_h$ form a basis of $\R^h$.
The assumption $h\le n-1$ ensures that the cardinalities
$c=1,\ldots,h$ used here are all available in
\eqref{eq:fixed-c-swap-span}.

Take any target matrix $Z\in\mathcal S$, and let
$Z_O=Z_{\cdot,O}\in\R^{h\times n}$ be its optional block.  Because
$\Gamma$ is
invertible, there is a unique coefficient matrix
\[
  \Xi=\Gamma^{-1}Z_O\in\R^{h\times n}
\]
whose rows we denote by $w_1^\top,\ldots,w_h^\top$.  Multiplying back by
$\Gamma$ gives
\begin{equation}\label{eq:swap-matrix-decomposition}
  Z_O=\Gamma\Xi=\sum_{c=1}^h \gamma_cw_c^\top.
\end{equation}
The zero-sum condition defining $\mathcal S$ is preserved when we pass from
$Z_O$ to $\Xi$:
\[
  \Xi\one_O=\Gamma^{-1}Z_O\one_O=0_h.
\]
The $c$th coordinate of $\Xi\one_O$ is $w_c^\top\one_O$; hence every
$w_c$ lies in $W_{\mathrm{swap}}$.  The core column of $Z$ is zero, so
restoring that column in \eqref{eq:swap-matrix-decomposition} gives
\[
  Z=\sum_{c=1}^h \gamma_c\widetilde w_c^\top.
\]
Each summand belongs to $D$ by Equation~\eqref{eq:fixed-c-swap-span}.
Since $Z\in\mathcal S$ was arbitrary, this proves the containment in
\eqref{eq:swap-subspace}.

For the dimension count, let
$W_{\mathrm{swap}}^h:=W_{\mathrm{swap}}\times\cdots\times
W_{\mathrm{swap}}$ denote the Cartesian product of $h$ copies of
$W_{\mathrm{swap}}$.
The linear map
\[
  Z\longmapsto\bigl(Z_{1,O}^\top,\ldots,Z_{h,O}^\top\bigr)
  \quad\text{from }\mathcal S\text{ to }W_{\mathrm{swap}}^h
\]
is an isomorphism.

Indeed, it is well defined because every optional row of
$Z\in\mathcal S$ belongs to $W_{\mathrm{swap}}$.  It is injective because
the optional rows determine $Z$: its remaining core column is zero by
definition of $\mathcal S$.  It is surjective because any tuple
$(u_1,\ldots,u_h)\in W_{\mathrm{swap}}^h$ produces a matrix in $\mathcal S$
by taking $u_m^\top$ as optional row $m$ and setting the core column to zero.
Consequently,
\begin{equation}\label{eq:swap-subspace}
  \mathcal S\subseteq D,
  \qquad \dim\mathcal S=h\dim W_{\mathrm{swap}}=h(n-1).
\end{equation}

\paragraph{Step 3d: Count the cardinality-change directions.}
Step 3c accounted for changes that keep $|C|$ fixed.  We now isolate the
directions obtained by changing $|C|$.  The argument has four parts: average
away the identities of the selected optional labels, rewrite each averaged
difference through a linear map, show that these differences span $h$
directions, and finally separate them from the label-swap space.

Embed the optional all-ones vector in coefficient space as
\[
  \widehat{\one}_O:=(0,\one_O^\top)^\top\in\R^{n+1}.
\]

\emph{3d(i): Average away the label identities.}
For $c=0,\ldots,n$, average $V_C$ over all $c$-subsets of $O$.  Since every
set in this average has cardinality $c$, the factor $\gamma_c$ is the same in
every summand.  Hence
\[
  \bar V_c
  =\gamma_c\left(
      e_0+\binom nc^{-1}\sum_{|C|=c}\widehat{\chi}_C
    \right)^\top.
\]
For a fixed optional label $j\in O$, exactly
$\binom{n-1}{c-1}$ of the $\binom nc$ sets of size $c$ contain $j$.
Thus, for $1\le c\le n$, the $j$th optional coordinate of the average
incidence vector is
\[
  \frac{\binom{n-1}{c-1}}{\binom nc}=\frac cn.
\]
For $c=0$, the average incidence vector is zero, which gives the same
formula.  Since this calculation is identical for every $j\in O$, we obtain
\begin{equation}\label{eq:averaged-matrix}
  \bar V_c
  =\binom nc^{-1}\sum_{|C|=c}V_C
  =\gamma_c\left(e_0+\frac cn\widehat{\one}_O\right)^\top.
\end{equation}
At $c=0$, there is only the empty set, so
$\bar V_0=V_\varnothing$.  It follows directly from
\eqref{eq:D-definition} that
\[
  \bar V_c-\bar V_0
  =\binom nc^{-1}\sum_{|C|=c}(V_C-V_\varnothing)
  \in D.
\]
Thus every averaged cardinality difference is already a direction in $D$.

\emph{3d(ii): Express each averaged difference through one vector.}
Expanding Equation~\eqref{eq:averaged-matrix} gives
\[
  \bar V_c-\bar V_0
  =(\gamma_c-\gamma_0)e_0^\top
   +\frac cn \gamma_c\widehat{\one}_O^\top.
\]
The first term depends on $\gamma_c-\gamma_0$, but the second apparently
depends on $c\gamma_c$.  The following identity expresses both terms using
the single vector $\gamma_c-\gamma_0$.

Let $\Lambda=\operatorname{diag}(t+1,\ldots,t+h)$.  In row $m$,
\[
  [\Lambda\gamma_c]_m
  =\frac{m+t}{m+t+c}
  =1-\frac{c}{m+t+c}
  =1-c\gamma_c(m),
\]
whereas $[\Lambda\gamma_0]_m=1$.  Therefore
\[
  \Lambda(\gamma_c-\gamma_0)=-c\gamma_c.
\]
This motivates the linear map
\[
  \Phi:\R^h\longrightarrow\R^{h\times(n+1)},
  \qquad
  \Phi(x)=x e_0^\top-\frac1n(\Lambda x)\widehat{\one}_O^\top.
\]
Its core column is $x$, and every optional column is
$-(1/n)\Lambda x$.  Substituting $x=\gamma_c-\gamma_0$ gives
\[
  \Phi(\gamma_c-\gamma_0)
  =(\gamma_c-\gamma_0)e_0^\top
   +\frac cn \gamma_c\widehat{\one}_O^\top.
\]
Comparison with the expanded averaged difference proves
\begin{equation}\label{eq:Phi-identity}
  \bar V_c-\bar V_0=\Phi(\gamma_c-\gamma_0).
\end{equation}

\emph{3d(iii): Obtain $h$ independent cardinality directions.}
Consider the matrix whose columns are the first $h$ difference vectors:
\[
  \Gamma^\Delta
  =[\gamma_1-\gamma_0\ \cdots\ \gamma_h-\gamma_0]
  \in\R^{h\times h}.
\]
Its entries are
\[
  \Gamma^\Delta_{m,c}=\gamma_c(m)-\gamma_0(m)
  =-\frac{c}{(m+t)(m+t+c)},
\]
and hence
\[
  \Gamma^\Delta=-\Lambda^{-1}\Gamma\Lambda_{\mathrm{card}},
\]
where
\[
  \Lambda_{\mathrm{card}}=\operatorname{diag}(1,\ldots,h).
\]
Both $\Lambda^{-1}$ and $\Lambda_{\mathrm{card}}$ are invertible, and the
Cauchy matrix $\Gamma$ was shown to be nonsingular in Step 3c.  Thus
$\Gamma^\Delta$ is nonsingular, so
$\gamma_1-\gamma_0,\ldots,\gamma_h-\gamma_0$ form a basis of $\R^h$.

Define
\[
  \mathcal H=\Phi(\R^h).
\]
For each $c=1,\ldots,h$, Equation~\eqref{eq:Phi-identity} and part 3d(i)
show that $\Phi(\gamma_c-\gamma_0)\in D$.  Since the vectors
$\gamma_c-\gamma_0$ form a basis of $\R^h$, any $x\in\R^h$ has coefficients
$\eta_1,\ldots,\eta_h$ such that
\[
  x=\sum_{c=1}^h\eta_c(\gamma_c-\gamma_0).
\]
By linearity and Equation~\eqref{eq:Phi-identity},
\[
  \Phi(x)
  =\sum_{c=1}^h\eta_c\Phi(\gamma_c-\gamma_0)
  =\sum_{c=1}^h\eta_c(\bar V_c-\bar V_0)
  \in D.
\]
Therefore $\mathcal H\subseteq D$.

Moreover, $\Phi$ is injective: if $\Phi(x)=0$, then its core column, which
equals $x$, must be zero.  Consequently
\[
  \dim\mathcal H=\dim\R^h=h.
\]

\emph{3d(iv): Separate cardinality changes from label swaps.}
Every matrix in $\mathcal S$ has zero core column.  On the other hand, the
core column of $\Phi(x)\in\mathcal H$ is $x$.  Thus, if
$Z\in\mathcal S\cap\mathcal H$, then $Z=\Phi(x)$ for some $x$, and the zero
core column of $Z$ forces $x=0$ and hence $Z=0$.  Therefore
\[
  \mathcal S\cap\mathcal H=\{0\}.
\]
We have therefore found $h$ cardinality-change directions in $D$ that are
linearly independent of the $h(n-1)$ label-swap directions in $\mathcal S$.

\paragraph{Step 3e: Show that there are no other directions.}
Steps 3c and 3d showed that $\mathcal S\subseteq D$,
$\mathcal H\subseteq D$, and $\mathcal S\cap\mathcal H=\{0\}$.  To finish
the dimension calculation, it remains to prove the reverse inclusion
$D\subseteq\mathcal S+\mathcal H$.  Since $D$ is defined as the span of the
matrices $V_C-V_\varnothing$, it is enough to decompose one such generator.

\emph{3e(i): Decompose each generator into two parts.}
Fix $C\subseteq O$ and write $c=|C|$.  Starting from the definition of
$V_C$, add and subtract the uniform optional-label term
$(c/n)\gamma_c\widehat{\one}_O^\top$:
\begin{equation}\label{eq:D-decomposition}
\begin{aligned}
  V_C-V_\varnothing
  &=\gamma_c(e_0+\widehat{\chi}_C)^\top-\gamma_0e_0^\top\\
  &=\gamma_c\left(\widehat{\chi}_C-
      \frac cn\widehat{\one}_O\right)^\top
    +(\gamma_c-\gamma_0)e_0^\top
    +\frac cn \gamma_c\widehat{\one}_O^\top\\
  &=\gamma_c\left(\widehat{\chi}_C-
      \frac cn\widehat{\one}_O\right)^\top
    +\Phi(\gamma_c-\gamma_0).
\end{aligned}
\end{equation}
The final equality is exactly the identity established in part 3d(ii).

\emph{3e(ii): Identify the two terms.}
Define
\[
  u_C=\widehat{\chi}_C-\frac cn\widehat{\one}_O.
\]
Both $\widehat{\chi}_C$ and $\widehat{\one}_O$ have zero core coordinate, so
$(u_C)_0=0$.
On the optional coordinates,
\[
  \sum_{j\in O}(u_C)_j
  =|C|-\frac cn|O|
  =c-\frac cn n
  =0.
\]
Consequently, the matrix $\gamma_cu_C^\top$ has zero core column and every row
has optional-coordinate sum zero.  By the definition of $\mathcal S$, this
shows
\[
  \gamma_cu_C^\top\in\mathcal S.
\]
The second term $\Phi(\gamma_c-\gamma_0)$ belongs to
$\mathcal H=\Phi(\R^h)$ by definition.  Equation~\eqref{eq:D-decomposition}
therefore places every generator $V_C-V_\varnothing$ in
$\mathcal S+\mathcal H$.  Taking their span gives
\[
  D\subseteq\mathcal S+\mathcal H.
\]

\emph{3e(iii): Prove equality and count the dimensions.}
The opposite inclusion follows from the earlier work:
Equation~\eqref{eq:swap-subspace} gives $\mathcal S\subseteq D$, and part
3d(iii) gives $\mathcal H\subseteq D$.  Because $D$ is a linear subspace,
these two containments imply
\[
  \mathcal S+\mathcal H\subseteq D.
\]
Combining the two inclusions yields $D=\mathcal S+\mathcal H$.  Part 3d(iv)
showed that $\mathcal S\cap\mathcal H=\{0\}$, so this sum is direct.
Using Equation~\eqref{eq:swap-subspace} and the identity
$\dim\mathcal H=h$ from part 3d(iii), we obtain
\begin{equation}\label{eq:D-direct-sum}
  D=\mathcal S\oplus\mathcal H,
  \qquad
  \dim D=h(n-1)+h=hn.
\end{equation}

\emph{3e(iv): Return from coefficient matrices to loss columns.}
For any finite family of vectors, its affine dimension is the dimension of
the span of its differences from one fixed member.  Here the fixed member
is the column indexed by $B_0=B_\varnothing$.  For every $C\subseteq O$,
the identities $L_{\cU,B_C}=\one_\cU-g_C$ and
$g_C=2\Psi(V_C)$ give
\[
  L_{\cU,B_C}-L_{\cU,B_0}
  =-(g_C-g_\varnothing)
  =-2\Psi(V_C-V_\varnothing).
\]
Because $D$ is the span of the matrices $V_C-V_\varnothing$ and $\Psi$ is
linear, this implies
\[
  \operatorname{span}\{L_{\cU,B_C}-L_{\cU,B_0}:C\subseteq O\}
  =-2\Psi(D).
\]
Multiplication by the nonzero scalar $-2$ does not change dimension, and
$\Psi$ is injective by Step 3b.  Therefore
\[
\begin{aligned}
  \affdim\{L_{\cU,B_C}:C\subseteq O\}
  &=\dim\operatorname{span}
    \{L_{\cU,B_C}-L_{\cU,B_0}:C\subseteq O\}\\
  &=\dim\Psi(D)\\
  &=\dim D\\
  &=hn.
\end{aligned}
\]

\paragraph{Step 4: Convert affine dimension into the CC-dimension bound.}
We now translate the $hn$ independent active loss differences into the
feasible-subspace lower bound \eqref{eq:ccdim-trigger-lower}.  The argument
has four parts: define the relevant distribution perturbations, characterize
the two-sided feasible subspace, count its dimension, and apply the bound.

\emph{4(i): Define the perturbation and identify the active constraints.}
Let $\mathcal Q=Q_{B_0}^L$, where $B_0=T$.  For any prediction $B$, write
\[
  d_B=L_{\cdot,B}-L_{\cdot,B_0}.
\]
The trigger set can then be written as
\[
  \mathcal Q=\{q\in\Delta_N:q^\top d_B\ge0
       \text{ for every prediction }B\}.
\]
A vector $v=(v_A)_{A\in\cY}\in\R^N$ is a \emph{two-sided feasible
direction} at $p$ if there is an $\varepsilon_0>0$ such that
\[
  p+\varepsilon v\in\mathcal Q
  \quad\text{and}\quad
  p-\varepsilon v\in\mathcal Q
  \qquad(0\le\varepsilon\le\varepsilon_0).
\]
Denote the subspace of all such directions by
$\operatorname{lin}_{\mathcal Q}(p)\subseteq\R^N$.
Thus $v_A$ is an infinitesimal change in the probability assigned to outcome
$A$.  This vector $v\in\R^N$ is unrelated to the optional-label vector
$w\in W_{\mathrm{swap}}\subseteq\R^O$ used in Step 3c.

The support of $p$ is exactly $\cU$, and $p_A>0$ for every $A\in\cU$.
The active loss-comparison inequalities at $p$ are precisely those for the
predictions $B_C$, $C\subseteq O$, because these predictions are tied with
$B_0$.  Every other comparison is strict by
\eqref{eq:strict-score-deficit}.

\emph{4(ii): Characterize the two-sided feasible directions.}
We derive one condition from each type of constraint.

\emph{Nonnegativity.}
If $A\notin\cU$, then $p_A=0$.  Two-sided feasibility requires
\[
  p_A+\varepsilon v_A=\varepsilon v_A\ge0,
  \qquad
  p_A-\varepsilon v_A=-\varepsilon v_A\ge0.
\]
Both inequalities can hold only when $v_A=0$.  Hence
$v_{\cU^c}=0$.  If $A\in\cU$, then $p_A>0$, so both signs remain
nonnegative for all sufficiently small $\varepsilon$.

\emph{Normalization.}
The vectors $p\pm\varepsilon v$ must have total mass one.  Since
$\one_N^\top p=1$, this is equivalent to $\one_N^\top v=0$.  Using
$v_{\cU^c}=0$, the condition becomes
$\one_\cU^\top v_\cU=0$.

\emph{Active Bayes ties.}
For $C\subseteq O$, the comparison with $B_C$ is active, so
$p^\top d_{B_C}=0$.  Feasibility in the two signs requires
\[
  (p\pm\varepsilon v)^\top d_{B_C}
  =\pm\varepsilon v^\top d_{B_C}\ge0.
\]
This holds for both signs exactly when $v^\top d_{B_C}=0$.  Because
$v_{\cU^c}=0$, this is equivalent to orthogonality of $v_\cU$ to the
restricted active difference vector
$L_{\cU,B_C}-L_{\cU,B_0}$.

We have proved that every two-sided feasible direction satisfies the three
conditions on the right-hand side of
\begin{equation}\label{eq:feasible-directions}
  \operatorname{lin}_{\mathcal Q}(p)
  =\left\{v\in\R^N:
    \begin{array}{l}
      v_{\cU^c}=0,\\
      \one_{\cU}^\top v_{\cU}=0,\\
      v_{\cU}\perp
      \operatorname{span}\{L_{\cU,B_C}-L_{\cU,B_0}:C\subseteq O\}
    \end{array}
  \right\}.
\end{equation}
It remains to verify the reverse inclusion in this display.  Suppose $v$
satisfies the three stated conditions.  The first two keep
$p\pm\varepsilon v$ in the simplex for all sufficiently small
$\varepsilon$: coordinates in $\cU^c$ remain zero, coordinates in $\cU$
remain nonnegative because they start strictly positive, and total mass
remains one.  Every active comparison remains an equality by the third
condition.  Each inactive comparison has a strictly positive margin at
$p$; since there are finitely many predictions, all of these margins remain
positive for both signs when $\varepsilon$ is sufficiently small.  Thus
Equation~\eqref{eq:feasible-directions} describes exactly the two-sided
feasible subspace.

\emph{4(iii): Count the dimension of this subspace.}
Let
\[
  E=\operatorname{span}
  \{L_{\cU,B_C}-L_{\cU,B_0}:C\subseteq O\}.
\]
Put
\[
  R=|\cU|=\|p\|_0,
  \qquad d_E=\dim E=hn,
\]
where the last equality is the conclusion of Step 3e.
After the condition $v_{\cU^c}=0$, the free vector $v_\cU$ lies in
$\R^R$.  Equation~\eqref{eq:feasible-directions} says precisely that
$v_\cU$ is orthogonal both to $E$ and to $\one_\cU$.  Consequently,
restriction to the coordinates in $\cU$ identifies the feasible subspace
with
\[
  \operatorname{lin}_{\mathcal Q}(p)
  \cong
  \bigl(E+\operatorname{span}\{\one_\cU\}\bigr)^\perp
  \subseteq\R^R.
\]

It remains to check that the normalization vector $\one_\cU$ contributes
one constraint independent of the $d_E$ constraints represented by $E$.
Indeed,
for every active prediction $B_C$,
\[
  p_\cU^\top(L_{\cU,B_C}-L_{\cU,B_0})
  =p^\top(L_{\cdot,B_C}-L_{\cdot,B_0})
  =0,
\]
because $B_C$ and $B_0$ are tied under $p$.  Hence the entire active
difference-vector span is orthogonal to $p_\cU$.  In contrast,
\[
  p_\cU^\top\one_\cU=1,
\]
so $\one_\cU$ is not in that span.  The $d_E$ active-difference constraints
and the one normalization constraint consequently have total rank $d_E+1$.
Taking the orthogonal complement in the $R$-dimensional space gives
\begin{equation}\label{eq:mu-calculation}
  \mu_{Q_{B_0}^L}(p)
  =\dim\operatorname{lin}_{\mathcal Q}(p)
  =R-(d_E+1)
  =R-d_E-1.
\end{equation}

\emph{4(iv): Apply the feasible-subspace lower bound.}
Substitute $\|p\|_0=R$ and
Equation~\eqref{eq:mu-calculation} into
\eqref{eq:ccdim-trigger-lower}:
\[
  \CCdim(L^{F_1})
  \ge \|p\|_0-\mu_{Q_{B_0}^L}(p)-1
  =R-(R-d_E-1)-1=d_E=hn.
\]

\paragraph{Step 5: Derive the stated numerical bounds.}
Finally,
\[
  t=\frac{s}{3}+O(1),
  \qquad n=\frac{2s}{3}+O(1),
  \qquad h=\frac{s}{\sqrt3}+O(1),
\]
which proves \eqref{eq:ccdim-asymptotic-lower}.  If $s\ge12$, then
$t\ge s/4$, $n\ge2s/3$, and
\[
  h\ge\sqrt{st}-1\ge\frac{s}{2}-1\ge\frac{s}{3},
\]
so $hn\ge2s^2/9$.
\end{proof}


\clearpage
\begin{thebibliography}{99}
\normalsize
\setlength{\baselineskip}{11pt}
\setlength{\itemsep}{2pt plus 1pt minus 1pt}
\setlength{\parsep}{0pt}

\bibitem[Agarwal and Agarwal(2015)]{agarwal2015property}
Arpit Agarwal and Shivani Agarwal.
\newblock On consistent surrogate risk minimization and property elicitation.
\newblock In \emph{Proceedings of the 28th Conference on Learning Theory},
  volume 40 of PMLR, pages 4--22, 2015.

\bibitem[Dembczy\'{n}ski et~al.(2011)Dembczy\'{n}ski, Waegeman, Cheng, and
  H\"ullermeier]{dembczynski2011exact}
Krzysztof~J. Dembczy\'{n}ski, Willem Waegeman, Weiwei Cheng, and Eyke
  H\"ullermeier.
\newblock An exact algorithm for F-measure maximization.
\newblock In \emph{Advances in Neural Information Processing Systems 24},
  pages 1404--1412, 2011.

\bibitem[Dembczy\'{n}ski et~al.(2013)Dembczy\'{n}ski, Jachnik, Kot\l{}owski,
  Waegeman, and H\"ullermeier]{dembczynski2013optimizing}
Krzysztof Dembczy\'{n}ski, Arkadiusz Jachnik, Wojciech Kot\l{}owski, Willem
  Waegeman, and Eyke H\"ullermeier.
\newblock Optimizing the F-measure in multi-label classification: Plug-in rule
  approach versus structured loss minimization.
\newblock In \emph{Proceedings of the 30th International Conference on Machine
  Learning}, volume 28 of PMLR, pages 1130--1138, 2013.

\bibitem[Finocchiaro et~al.(2020)Finocchiaro, Frongillo, and
  Waggoner]{finocchiaro2020embedding}
Jessie Finocchiaro, Rafael Frongillo, and Bo Waggoner.
\newblock Embedding dimension of polyhedral losses.
\newblock In \emph{Proceedings of the 33rd Conference on Learning Theory},
  volume 125 of PMLR, pages 1558--1585, 2020.

\bibitem[Finocchiaro et~al.(2021)Finocchiaro, Frongillo, and
  Waggoner]{finocchiaro2021lower}
Jessica Finocchiaro, Rafael~M. Frongillo, and Bo Waggoner.
\newblock Unifying lower bounds on prediction dimension of convex surrogates.
\newblock In \emph{Advances in Neural Information Processing Systems 34},
  pages 22046--22057, 2021.

\bibitem[Finocchiaro et~al.(2024)Finocchiaro, Frongillo, and
  Waggoner]{finocchiaro2024framework}
Jessie Finocchiaro, Rafael~M. Frongillo, and Bo Waggoner.
\newblock An embedding framework for the design and analysis of consistent
  polyhedral surrogates.
\newblock \emph{Journal of Machine Learning Research}, 25(63):1--60, 2024.

\bibitem[Frongillo and Kash(2015)]{frongillo2015elicitation}
Rafael Frongillo and Ian~A. Kash.
\newblock On elicitation complexity.
\newblock In \emph{Advances in Neural Information Processing Systems 28},
  pages 3258--3266, 2015.

\bibitem[Koyejo et~al.(2015)Koyejo, Natarajan, Ravikumar, and
  Dhillon]{koyejo2015consistent}
Oluwasanmi~O. Koyejo, Nagarajan Natarajan, Pradeep~K. Ravikumar, and
  Inderjit~S. Dhillon.
\newblock Consistent multilabel classification.
\newblock In \emph{Advances in Neural Information Processing Systems 28},
  pages 3321--3329, 2015.

\bibitem[Nowak et~al.(2019)Nowak, Bach, and Rudi]{nowak2019sharp}
Alex Nowak, Francis Bach, and Alessandro Rudi.
\newblock Sharp analysis of learning with discrete losses.
\newblock In \emph{Proceedings of the Twenty-Second International Conference
  on Artificial Intelligence and Statistics}, volume 89 of PMLR,
  pages 1920--1929, 2019.

\bibitem[Ramaswamy and Agarwal(2012)]{ramaswamy2012classification}
Harish~G. Ramaswamy and Shivani Agarwal.
\newblock Classification calibration dimension for general multiclass losses.
\newblock In \emph{Advances in Neural Information Processing Systems 25},
  pages 2087--2095, 2012.

\bibitem[Ramaswamy and Agarwal(2016)]{ramaswamy2016ccdim}
Harish~G. Ramaswamy and Shivani Agarwal.
\newblock Convex calibration dimension for multiclass loss matrices.
\newblock \emph{Journal of Machine Learning Research}, 17(14):1--45, 2016.

\bibitem[Ramaswamy et~al.(2013)Ramaswamy, Agarwal, and
  Tewari]{ramaswamy2013lowrank}
Harish~G. Ramaswamy, Shivani Agarwal, and Ambuj Tewari.
\newblock Convex calibrated surrogates for low-rank loss matrices with
  applications to subset ranking losses.
\newblock In \emph{Advances in Neural Information Processing Systems 26},
  pages 1475--1483, 2013.

\bibitem[Ramaswamy et~al.(2014)Ramaswamy, Srinivasan~Babu, Agarwal, and
  Williamson]{ramaswamy2014output}
Harish~G. Ramaswamy, Balaji Srinivasan~Babu, Shivani Agarwal, and Robert~C.
  Williamson.
\newblock On the consistency of output code based learning algorithms for
  multiclass learning problems.
\newblock In \emph{Proceedings of the 27th Conference on Learning Theory},
  volume 35 of PMLR, pages 885--902, 2014.

\bibitem[Waegeman et~al.(2014)Waegeman, Dembczy\'{n}ski, Jachnik, Cheng, and
  H\"ullermeier]{waegeman2014bayes}
Willem Waegeman, Krzysztof Dembczy\'{n}ski, Arkadiusz Jachnik, Weiwei Cheng,
  and Eyke H\"ullermeier.
\newblock On the Bayes-optimality of F-measure maximizers.
\newblock \emph{Journal of Machine Learning Research}, 15(103):3513--3568,
  2014.

\bibitem[Wang(2026)]{wang2026revisiting}
Zixun Wang.
\newblock Revisiting F-measure optimization in multi-label classification:
  A sampling-based approach.
\newblock In \emph{Proceedings of the IEEE/CVF Conference on Computer Vision
  and Pattern Recognition}, pages 16845--16854, 2026.

\bibitem[Ye et~al.(2012)Ye, Chai, Lee, and Chieu]{ye2012optimizing}
Nan Ye, Kian Ming~A. Chai, Wee Sun Lee, and Hai Leong Chieu.
\newblock Optimizing F-measures: A tale of two approaches.
\newblock In \emph{Proceedings of the 29th International Conference on
  Machine Learning}, pages 289--296, 2012.

\bibitem[Zhang et~al.(2020)Zhang, Ramaswamy, and Agarwal]{zhang2020convex}
Mingyuan Zhang, Harish~G. Ramaswamy, and Shivani Agarwal.
\newblock Convex calibrated surrogates for the multi-label F-measure.
\newblock In \emph{Proceedings of the 37th International Conference on Machine
  Learning}, volume 119 of PMLR, pages 11246--11255, 2020.

\end{thebibliography}
\end{document}